%% file: rtknnc_KAIS.tex
\documentclass[pdflatex,sn-basic,Numbered]{sn-jnl}

\usepackage{graphicx}
\usepackage{multirow}
\usepackage{amsmath,amssymb,amsfonts}
\usepackage{amsthm}
\usepackage{mathrsfs}
\usepackage[title]{appendix}
\usepackage{xcolor}
\usepackage{textcomp}
\usepackage{booktabs}
\usepackage{algorithm}
\usepackage{algorithmicx}
\usepackage{algpseudocode}
\usepackage{listings}
\usepackage{bm}

\newtheorem{theorem}{Theorem}
\newtheorem{proposition}[theorem]{Proposition}

\newcommand{\R}{\mathbb{R}}

\newcommand{\calG}{\mathcal{G}}
\newcommand{\calN}{\mathcal{N}}
\newcommand{\calP}{\mathcal{P}}
\newcommand{\calR}{\mathcal{R}}

\newcommand{\calV}{\mathcal{V}}
\newcommand{\KNN}{\operatorname{KNN}}
\newcommand{\RNN}{\operatorname{RNN}}

\newcommand{\argmin}{\operatorname*{arg\,min}}

\graphicspath{{./}}

\begin{document}

\title[Round-Trip KNN Clustering]{Round-Trip KNN Clustering: multiscale hierarchical cluster detection on directed nearest-neighbour graphs}

\author*[1]{\fnm{Eraldo P.} \sur{Marinho}}\email{pereira.marinho@unesp.br}
\author[1]{\fnm{Caetano Mazzoni} \sur{Ranieri}}\email{cm.ranieri@unesp.br}
\author[1]{\fnm{Fabricio A.} \sur{Breve}}\email{fabricio.breve@unesp.br}

\affil*[1]{\orgdiv{Department of Statistics, Applied Mathematics and Computing},
\orgname{S\~ao Paulo State University (UNESP), Institute of Geosciences and Exact Sciences},
\orgaddress{\street{Av. 24-A, 1515}, \city{Rio Claro}, \state{S\~ao Paulo}, \country{Brazil}}}

\abstract{
We introduce Round-Trip KNN Clustering (RTKNNC), a graph-based method for
finding cluster structure at several neighbourhood scales without requiring the
number of clusters in advance. Unlike approaches that first make a
$k$-nearest-neighbour (KNN) graph undirected, RTKNNC keeps both directions of
the neighbour relation: which points a given point selects and which points
select it. Incoming selections are treated as weighted votes that help decide
which local connections remain visible during a recursive forward-and-reverse
traversal. Repeating the procedure for increasing $K$ reveals how groups persist
or merge as the neighbourhood scale grows; for the reference inverse-square
model before structural refinement, clusters can merge but do not split.
Because graph connectivity can occasionally join distinct groups through a
sparse bridge or a small region of overlap, we add an optional label-free
refinement. It first tests whether an already formed component is better
described by two or three Gaussian subpopulations, and accepts a subdivision
only when the proposed groups are large enough and consistent with the visible
KNN graph. Across eight synthetic datasets and $K=2,\ldots,16$, independent C
and Python implementations produced identical partitions in all 120 reference
runs. Refinement increased adjusted Rand index from $0.7817$ to $0.9627$ on a
variable-density benchmark and from $0.8083$ to $0.9853$ on a sparse-bridge
benchmark. Comparisons with seven external clustering methods show competitive
performance while preserving a label-free cluster-construction process.
}

\maketitle

\section{Introduction}\label{sec:introduction}

Nearest-neighbour relations provide a simple way to describe local structure in
data. They are used in classification, density estimation, graph construction,
manifold learning, outlier detection, and clustering
\citep{CoverHart1967,LoftsgaardenQuesenberry1965,Bentley1975,FriedmanBentleyFinkel1977,VonLuxburg2007}.
For clustering in particular, local neighbourhood information is attractive
because it can reveal non-convex or unevenly populated groups that are poorly
described by a centroid or by a single global distance threshold. Existing
approaches exploit shared neighbours, mutual or reciprocal KNN graphs, reverse
nearest neighbours, and density-based connectivity
\citep{JarvisPatrick1973,Ertoz2003SNN,BritoChavezQuirozYukich1997,GowdaKrishna1978,ChowdhuryAmorim2018,XieLeeWangChenZhou2019}.

A directed KNN graph contains two complementary views of every point. Its
outgoing links tell us which points it regards as neighbours; its incoming
links tell us which other points regard it as a neighbour. These views are not
generally equivalent, and making the graph undirected at the outset discards
that asymmetry. RTKNNC is built around keeping it. Each outgoing KNN selection
casts a weighted vote, and the votes received by a point provide a local measure
of how strongly it is supported by its neighbourhood. This support is then used
to determine which nearby links are visible during cluster formation.

For a fixed value of $K$, the method can be described intuitively as an
expansion process. Starting from a seed, RTKNNC follows visible links both
forward, to selected neighbours, and backward, to points that selected the
current point. Every newly reached point is explored in the same way, and the
process stops when no additional points can be added. We call this a
\emph{round trip}, and the final set a saturated round-trip closure. Clusters
therefore arise from local graph connectivity rather than from a prescribed
geometric shape. The number of clusters is not supplied to the algorithm.

The neighbourhood size $K$ also provides a natural notion of scale. Small
values describe very local structure, whereas larger values allow information
to propagate over broader neighbourhoods. Running the method independently for
increasing $K$ produces a sequence of cluster layers that makes persistent
groups and mergers directly visible. In this paper, \emph{hierarchical} refers
to this organisation across neighbourhood scales, not to a conventional
linkage-based agglomerative procedure. For the reference inverse-square voting
rule, before the optional structural refinement is applied, we show that these
layers have a useful merge-only property whenever they are proper partitions:
as $K$ increases, a cluster may persist or merge with another one, but it does
not split.

Any clustering method based on connectivity faces a familiar difficulty: a
small number of links can join groups that are otherwise well separated. A
sparse chain of points may act as a bridge, or two populations of different
densities may overlap in a narrow region. RTKNNC therefore includes an optional
structural-refinement stage that is applied only after a connected component
has already been formed. The purpose of this stage is not to replace the
round-trip method with a Gaussian-mixture clustering algorithm. Instead, a
Gaussian-mixture model is used locally as a test of whether the existing
component contains evidence for two or three internal populations. The proposed
subdivision is accepted only when the statistical evidence is strong, every
resulting group is large enough, and the groups remain coherent in the visible
KNN graph. Reference labels are never used to make this decision, and the KNN
adjacency lists themselves are not rewritten.

This separation is important. The round-trip graph determines how local
evidence propagates and therefore produces the original cluster candidates; the
optional refinement is a conservative correction for cases in which connectivity
alone has joined distinguishable populations. It also lets us study two
questions separately: how vote allocation changes the visible graph, and when a
connected component needs an additional structural check.

KNN graphs, shared-neighbour evidence, and reciprocal or reverse-neighbour
relations are established ingredients
\citep{JarvisPatrick1973,Ertoz2003SNN,BritoChavezQuirozYukich1997,ChowdhuryAmorim2018,XieLeeWangChenZhou2019}.
The contribution of RTKNNC is to combine them into a directed round-trip
process whose behaviour can be followed across neighbourhood scales. Recent
KNNG-CS work is related in its use of weighted incoming KNN links to define node
importance \citep{LiuEtAl2026KNNGCS}; there, however, the quantity supports
greedy coreset selection, whereas RTKNNC uses incoming votes to regulate graph
traversal and cluster formation.

The main contributions of this paper are:
\begin{enumerate}
\item a clustering procedure that retains the direction of KNN relations and
combines forward and reverse traversal with weighted incoming votes and local
visibility;
\item a multiscale construction in which cluster structure is followed as $K$
changes, together with a proof of merge-only coarsening for the unrefined
reference inverse-square model;
\item a controlled family of vote-allocation rules---uniform,
inverse-distance, inverse-square, and sigmoid---that allows their effects to be
compared under the same per-source vote budget;
\item an optional label-free refinement that tests already formed components for
supported internal subdivisions and accepts them only after size and graph
coherence checks; and
\item an experimental evaluation covering implementation agreement, vote-model
ablations, multiscale behaviour, structural-refinement cases, and comparison
with seven external clustering methods.
\end{enumerate}

The remainder of the paper is organised as follows. Section~\ref{sec:related}
reviews related neighbour- and graph-based clustering approaches.
Section~\ref{sec:preliminaries} introduces the notation needed to describe the
directed KNN graph and the round-trip relation. Section~\ref{sec:method} presents
the voting, visibility, expansion, refinement, and multiscale construction.
Section~\ref{sec:experimental_protocol} describes the datasets, implementation
checks, parameter sweeps, and comparison protocol. Section~\ref{sec:results}
reports the results, Section~\ref{sec:discussion} discusses the behaviour and
limitations of the method, and Section~\ref{sec:conclusion} summarises the main
findings.

\section{Related work}\label{sec:related}

This section situates RTKNNC within the main clustering approaches that share its nearest-neighbour and graph-based foundations. We review shared-neighbour methods, mutual, reciprocal, and reverse nearest-neighbour constructions, and density-based and spectral graph clustering, highlighting the distinctions between these approaches and the visibility-filtered round-trip closure proposed here.

\subsection{Shared-neighbour and Jarvis--Patrick clustering}

Jarvis and Patrick introduced a clustering criterion based on shared near neighbours \citep{JarvisPatrick1973}. Later shared-nearest-neighbour (SNN) clustering methods used common neighbour counts to improve robustness to clusters of different shapes, sizes, and densities \citep{Ertoz2003SNN}. RTKNNC is related to this family because points sharing an outgoing neighbour belong to that neighbour's reverse-neighbour set. Rather than using this evidence as a static similarity score, RTKNNC embeds it in a recursive graph-closure process.

\subsection{Mutual, reciprocal, and reverse nearest-neighbour methods}

Mutual KNN and reciprocal nearest-neighbour graphs have been used for clustering and outlier detection \citep{BritoChavezQuirozYukich1997,GowdaKrishna1978,SardanaBhatnagar2014}. Reverse nearest-neighbour methods use incoming KNN links as density or representativeness signals \citep{ChowdhuryAmorim2018}. RTKNNC uses reverse-neighbour sets both for backward traversal and to compute kernel-weighted popularity. Popularity then modulates the local visibility threshold for forward and backward transitions.

Recent work continues to develop neighbour relations as clustering primitives.
MGNR constructs task-adapted multi-granularity relations
\citep{XieEtAl2024MGNR}; weighted-KNN density peaks use a point's propensity to
occur in other observations' KNN sets as density evidence
\citep{DingEtAl2024WeightedKNN}; and KNEG-CL evolves an unweighted directed KNN
graph with $K$ and uses in- and out-degree patterns for density-peak detection
and subcluster merging \citep{FeiEtAl2025KNEGCL}. These approaches pursue
multi-granularity representation, density estimation, and degree-driven
merging, respectively, whereas RTKNNC builds visibility-filtered recursive
closures and studies their evolution as cluster layers.

KNNG-CS is particularly close in its definition of popularity \citep{LiuEtAl2026KNNGCS}. It assigns every outgoing KNN edge a source-normalised softmax weight and defines node importance as the sum of incoming weights. KNNG-CS computes this quantity once, greedily selects coreset representatives, and removes each selected point together with its reverse neighbours. RTKNNC instead uses distance-weighted incoming votes to modulate visibility during recursive forward--backward closure. KNNG-CS fixes each source's outgoing mass to one; RTKNNC retains the geometry-dependent reference budget
\[
B_K(x_i)=\sum_{x_j\in\calN_K^+(x_i)}1/r_{ij}^2.
\] This common budget recovers the inverse-square operator exactly and permits comparison of alternative allocation shapes.

\subsection{Density-based and spectral graph clustering}

DBSCAN, OPTICS, and HDBSCAN define clusters through density connectivity and hierarchical density estimates \citep{Ester1996DBSCAN,Ankerst1999OPTICS,Campello2013HDBSCAN,CampelloMoulaviZimekSander2015}. Spectral clustering constructs a similarity graph and uses eigenvectors of the associated graph Laplacian to reveal cluster structure \citep{NgJordanWeiss2002,VonLuxburg2007}. The method proposed here is graph-based, but it is combinatorial and closure-driven: it does not require eigenvalue decomposition and it exposes a sequence of partitions as the neighbourhood size $K$ varies.
The comparison also includes classical agglomerative clustering
\citep{MurtaghContreras2012}, $k$-means \citep{Lloyd1982,MacQueen1967}, and
Gaussian mixtures, using the scikit-learn implementations
\citep{PedregosaEtAl2011}.

\section{Preliminaries}\label{sec:preliminaries}

This section introduces the notation and graph-theoretic constructions used throughout the paper. We first state the round-trip mechanism for a general directed relation and then specialise it to directed KNN graphs and reverse neighbourhoods.

\subsection{Relational basis of the round-trip operator}\label{sec:relational_origin}

The primitive idea can be stated for an arbitrary directed binary relation on
a common finite set.  Let
\begin{equation}
X=\{x_1,\ldots,x_N\}
\end{equation}
be a finite set and let
\begin{equation}
T\subseteq X\times X
\end{equation}
be a directed binary relation.  Its symmetrised forward--backward, or
round-trip, relation is
\begin{equation}
R_T=T\cup T^{-1}\subseteq X\times X.
\end{equation}
For $x\in X$, the corresponding one-step round-trip neighbourhood is
\begin{equation}
R_T(x)=\{y\in X:(x,y)\in T\}\cup\{y\in X:(y,x)\in T\}.
\label{eq:relational_roundtrip}
\end{equation}
Thus, one step retrieves both the objects selected directly by $x$ and the
objects that select $x$.

For a subset $S\subseteq X$, define
\begin{equation}
R_T(S)=\bigcup_{x\in S}R_T(x).
\end{equation}
The cumulative round-trip closure starts from
\begin{equation}
C_T^0(S)=S,
\end{equation}
and evolves according to
\begin{equation}
C_T^{t+1}(S)=C_T^t(S)\cup R_T(C_T^t(S)),\qquad t\geq 0.
\label{eq:relational_iteration}
\end{equation}
The process saturates when
\begin{equation}
C_T^{s+1}(S)=C_T^s(S)
\end{equation}
for the first time.  The integer $s$ is the saturation depth.  Because $X$ is
finite and the sequence is monotone, saturation is guaranteed; the saturated
set is the union of the weakly connected components of $(X,T)$ that intersect
$S$.

This relational view exactly describes the forward--backward mechanism used by
the proposed graph operator.  The algorithmic construction below instantiates
$T$ as the visibility-filtered edge relation of a directed KNN graph.

\subsection{Directed KNN graphs and reverse neighbourhoods}\label{sec:knn_graph}

Let
\begin{equation}
X=\{x_1,\ldots,x_N\}\subset\R^d
\end{equation}
be a finite dataset with a metric $\delta:X\times X\to\R_{\geq 0}$.
Assume $N\geq2$ and $1\leq K\leq N-1$. The experiments use Euclidean
distance on the stated feature columns. For each point $x_i$, let
\begin{equation}
\calN_K^+(x_i)=\KNN_K(x_i)
\end{equation}
denote the set of its $K$ nearest proper neighbours, which excludes $x_i$ itself.
The directed KNN graph is
\begin{equation}
\calG_K=(X,E_K),
\end{equation}
where
\begin{equation}
(x_i,x_j)\in E_K \quad\Longleftrightarrow\quad x_j\in\calN_K^+(x_i).
\end{equation}
The reverse-neighbour set of $x_j$ is
\begin{equation}
\calN_K^-(x_j)=\RNN_K(x_j)=\{x_i\in X:x_j\in\calN_K^+(x_i)\}.
\label{eq:rnn}
\end{equation}

\begin{figure}[t]
\centering
\includegraphics[width=4.8in,height=6.8in,keepaspectratio]{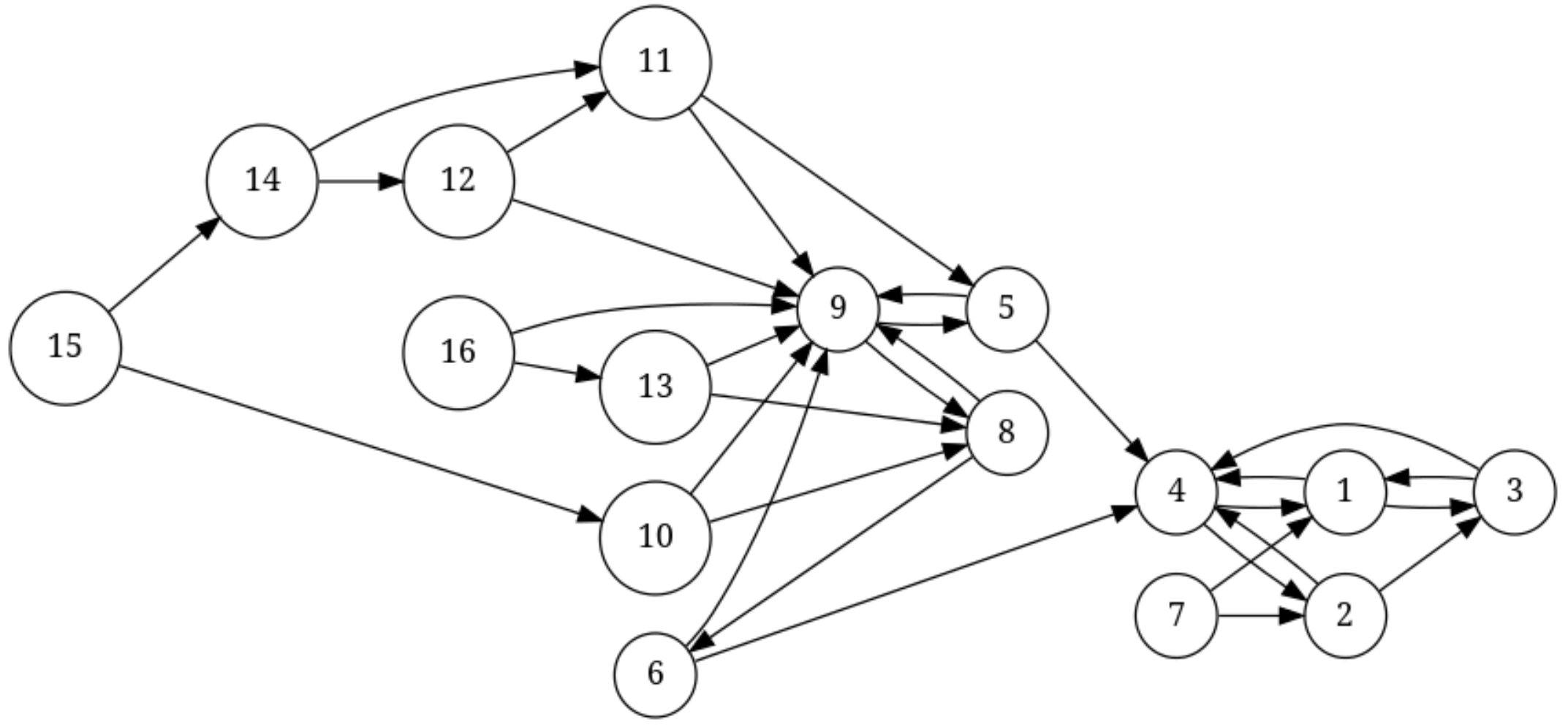}
\caption{Example of a directed KNN graph with $K=2$, used to illustrate the
proposed mechanism. Arrows represent directed nearest-neighbour choices. The
graph encodes neighbourhood connectivity, but its abstract layout does not
reproduce the spatial morphology of the original point distribution. Metric
distances are stored as edge attributes and are not represented by the drawn
edge lengths. The same graph supports reverse-neighbour retrieval and
visibility-filtered round-trip expansion.}\end{figure}

\section{Proposed algorithm}\label{sec:method}
\label{sec:proposed-algorithm}

This section develops RTKNNC from the directed KNN representation introduced above. We define the vote-kernel popularity and visibility rules, the resulting recursive round-trip closure, the optional structural-pruning stage, and the multiscale cluster-layer representation used in the experiments.

At a high level, for a fixed $K$, RTKNNC first builds a directed KNN graph.
Incoming neighbour selections define a popularity score, which adjusts a local
distance threshold and determines which KNN links may be traversed. Starting
from a seed, the algorithm repeatedly follows admissible links in both
directions until no new points are reached. Repeating this process for
increasing values of $K$ produces the multiscale cluster sequence. An optional
structural stage can then subdivide an overly connected component when
statistical and graph-based evidence support the split.

\subsection{Vote-kernel popularity, visibility, and round-trip closure}\label{sec:visibility_closure}

The reference formulation adopts a popularity-modulated, visibility-filtered round-trip expansion. A point need not propagate through all of its $K$ nearest-neighbour links: only links admitted by its local visibility threshold participate in the forward--backward closure. To reproduce the reference implementation exactly, linear and squared distances must be distinguished.

\subsubsection{Linear and squared-distance representation}

Let
\begin{equation}
r_{ij}=\delta(x_i,x_j),
\qquad
D_{ij}=r_{ij}^{2}.
\label{eq:linear_squared_distance}
\end{equation}
The C implementation stores $D_{ij}$, rather than $r_{ij}$, in each KNN list. After the list is sorted, its stored radius is the largest retained squared distance,
\begin{equation}
R_K^{(2)}(x_i)=D_{i,(K)}
=\max_{x_j\in\calN_K^+(x_i)}D_{ij}.
\label{eq:squared_knn_radius}
\end{equation}
The corresponding linear scale is
\begin{equation}
R_K(x_i)=\sqrt{R_K^{(2)}(x_i)}=r_{i,(K)}.
\label{eq:linear_knn_radius}
\end{equation}

\subsubsection{Vote kernels and reverse-neighbour popularity}

For each source point $x_i$, define the dimensionless neighbour coordinate
\begin{equation}
u_{ij}=\frac{r_{ij}}{R_K(x_i)},
\qquad x_j\in\calN_K^+(x_i),
\label{eq:normalised_neighbour_distance}
\end{equation}
and compare the four allocation shapes
\begin{equation}
h(u)\in
\left\{
1,\frac{1}{u},\frac{1}{u^2},S_\alpha(u)
\right\},
\label{eq:vote_kernel_family}
\end{equation}
where
\begin{equation}
S_\alpha(u)=\frac{1}{1+\exp[\alpha(u-1)]}.
\label{eq:sigmoid_vote}
\end{equation}
Here, $\alpha>0$ controls the slope, and $R_K(x_i)$ supplies the local distance scale for every shape.  Directly comparing the raw values $1$, $1/r$, $1/r^2$, and $S_\alpha(r_{ij}/R_K(x_i))$ would confound shape with units and arbitrary outgoing amplitude.  We therefore retain the reference inverse-square vote mass of source $i$ as the common budget
\begin{equation}
B_K(x_i)=\sum_{x_k\in\calN_K^+(x_i)}\frac{1}{r_{ik}^2},
\label{eq:reference_source_budget}
\end{equation}
and redistribute that budget according to the selected shape:
\begin{equation}
g_{i,h}(x_j)=B_K(x_i)
\frac{h(u_{ij})}
{\displaystyle\sum_{x_k\in\calN_K^+(x_i)}h(u_{ik})}.
\label{eq:budget_matched_vote}
\end{equation}
Thus
$\sum_{x_j\in\calN_K^+(x_i)}g_{i,h}(x_j)=B_K(x_i)$
for every model. For a fixed allocation shape $h$, let
$g_h=\{g_{i,h}\}_{i=1}^{N}$ denote the resulting
vote model; when the selected shape is clear, we write $g=g_h$.

The term \emph{vote kernel} is used here to denote a non-negative local
allocation rule; it is not a Mercer kernel, and neither symmetry nor
positive semidefiniteness is assumed.

The corresponding incoming popularity of $x_j$ is
\begin{equation}
P_{K,g}(x_j)=
\sum_{x_i:\,x_j\in\calN_K^+(x_i)}g_{i,h}(x_j).
\label{eq:reverse_popularity}
\end{equation}
Thus, $P_{K,g}(x_j)$ is a kernel-weighted in-degree: it accumulates the
votes received by $x_j$ along its incoming KNN edges.
For $h(u)=1/u^2$, Eq.~\eqref{eq:budget_matched_vote} reduces algebraically to the reference edge vote,
\begin{equation}
g_{i,\mathrm{inv}\_r2}(x_j)
=B_K(x_i)\frac{(R_K(x_i)/r_{ij})^2}
{\sum_{x_k\in\calN_K^+(x_i)}(R_K(x_i)/r_{ik})^2}
=\frac{1}{r_{ij}^2}=\frac{1}{D_{ij}}.
\label{eq:reference_inverse_square_vote}
\end{equation}
The implementation uses this final expression directly for the inverse-square branch, avoiding a normalise--denormalise round-off cycle. Exact duplicate observations are excluded because the shared reference budget is undefined at $r=0$.

\subsubsection{RMS visibility in squared-distance units}

The base, or unmodulated, visibility threshold depends only on the outgoing
KNN distances of $x_i$ and is the RMS of the stored squared distances,
\begin{equation}
V_{K,0}^{(2)}(x_i)=
\left(
\frac{1}{K}
\sum_{x_j\in\calN_K^+(x_i)}D_{ij}^{2}
\right)^{1/2}.
\label{eq:rms_visibility_base}
\end{equation}
Here, the subscript $0$ denotes the absence of popularity modulation, rather
than an additional vote model, and the superscript $(2)$ indicates
squared-distance units.
Popularity then produces the effective threshold
\begin{equation}
V_{K,g}^{(2)}(x_i)=V_{K,0}^{(2)}(x_i)
\left(1+\frac{P_{K,g}(x_i)R_K^{(2)}(x_i)}{N}\right).
\label{eq:popularity_modulated_visibility}
\end{equation}
The subscript $g$ records that this effective threshold uses the popularity
computed under vote model $g$.  Since $P_{K,g}(x_i)\geq0$, modulation cannot
reduce the base threshold.
Equations~\eqref{eq:reference_inverse_square_vote}--\eqref{eq:popularity_modulated_visibility} reproduce the reference inverse-square implementation exactly. The comparison models change only the within-source allocation in Eq.~\eqref{eq:budget_matched_vote}. Under a global coordinate rescaling $r\mapsto cr$, the dimensionless
allocation coordinates $u_{ij}$ and shapes $h(u_{ij})$ remain unchanged,
whereas all four budget-matched edge votes $g_{i,h}(x_j)$ scale as
$c^{-2}$ because $B_K\mapsto c^{-2}B_K$. Since
$R_K^{(2)}\mapsto c^2R_K^{(2)}$, the popularity factor
$P_{K,g}R_K^{(2)}$ is invariant for every model. Feature-wise preprocessing can still change the KNN graph and must therefore be specified and reused exactly.

\subsubsection{Visible forward and backward steps}

For a fixed vote model $g$, define the directed visible-edge set
\begin{equation}
E^{\mathrm{vis}}_{K,g}
=
\left\{(x_i,x_j)\in E_K:
D_{ij}\leq V_{K,g}^{(2)}(x_i)
\right\}.
\label{eq:visible_edge_set}
\end{equation}
The corresponding undirected edge set obtained by disregarding orientation is
\begin{equation}
\overline E^{\mathrm{vis}}_{K,g}
=
\left\{
\{x_i,x_j\}:
(x_i,x_j)\in E^{\mathrm{vis}}_{K,g}
\ \text{or}\
(x_j,x_i)\in E^{\mathrm{vis}}_{K,g}
\right\}.
\label{eq:symmetrised_visible_edge_set}
\end{equation}
For any $C\subseteq X$, let
\begin{equation}
\overline E^{\mathrm{vis}}_{K,g}[C]
=
\left\{
\{x_i,x_j\}\in\overline E^{\mathrm{vis}}_{K,g}:
x_i,x_j\in C
\right\}
\label{eq:induced_visible_edge_set}
\end{equation}
denote the visible undirected edge set induced by $C$.

For a fixed vote model $g$, we omit the subscript $g$ from the notation
below whenever there is no ambiguity. The visible forward neighbourhood of $x_i$ is
\begin{equation}
\calV_K^+(x_i)=
\left\{
x_j\in\calN_K^+(x_i):
D_{ij}\leq V_{K,g}^{(2)}(x_i)
\right\}.
\label{eq:visible_forward}
\end{equation}
The visible backward neighbourhood of $x_i$ is
\begin{equation}
\calV_K^-(x_i)=
\left\{
x_\ell\in\calN_K^-(x_i):
D_{\ell i}\leq V_{K,g}^{(2)}(x_\ell)
\right\}.
\label{eq:visible_backward}
\end{equation}
The backward condition is evaluated from the viewpoint of the reverse
neighbour $x_\ell$: a reverse-neighbour step is accepted only if
$D_{\ell i}$ does not exceed the squared-distance visibility threshold of
$x_\ell$. The one-step visibility-filtered round-trip neighbourhood is
therefore
\begin{equation}
\calR_K^{\mathrm{vis}}(x_i)
=
\calV_K^+(x_i)\cup\calV_K^-(x_i).
\label{eq:visible_roundtrip}
\end{equation}

\subsubsection{Recursive closure and singleton rejection}

Starting from a seed point $x_i$, the recursive closure is
\begin{equation}
C_K^0(x_i)=\{x_i\},
\end{equation}
\begin{equation}
C_K^{t+1}(x_i)
=
C_K^t(x_i)\cup
\bigcup_{x_q\in C_K^t(x_i)}
\calR_K^{\mathrm{vis}}(x_q).
\label{eq:visibility_closure}
\end{equation}
The process saturates when
\begin{equation}
C_K^{t+1}(x_i)=C_K^t(x_i).
\end{equation}
The saturated closure is denoted
\begin{equation}
C_K^\infty(x_i)=C_K^s(x_i),
\end{equation}
where $s$ is the first saturation time.

A saturated closure is accepted as a cluster only when it contains more than
one point:
\begin{equation}
|C_K^\infty(x_i)|>1.
\label{eq:singleton_rejection}
\end{equation}
This rule is retained as an implementation safeguard. Under the stated
RMS visibility rule, however, every point has at least one visible proper
neighbour. Indeed,
\begin{equation}
D_{i,(1)}\leq
\left(\frac{1}{K}\sum_{j=1}^{K}D_{i,(j)}^2\right)^{1/2}
=V_{K,0}^{(2)}(x_i)\leq V_{K,g}^{(2)}(x_i).
\label{eq:nonempty_visible_neighbourhood}
\end{equation}
Consequently, no saturated closure is a singleton in the mathematical
formulation. Together with Proposition~\ref{prop:closure_component}, this
implies that the distinct closures form a partition of $X$. The empirical
coverage and disjointness checks below therefore validate implementation
agreement with this property.

\begin{proposition}[Closure--component equivalence]
\label{prop:closure_component}
For fixed $K$ and vote model $g$, the unpruned saturated closure
$C_K^\infty(x_i)$ is the vertex set of the connected component containing
$x_i$ in the symmetrised visible graph
$(X,\overline E^{\mathrm{vis}}_{K,g})$.
\end{proposition}

\begin{proof}
A forward step from $x_q$ reaches exactly the endpoints of visible directed
edges leaving $x_q$, whereas a backward step reaches exactly the initial
vertices of visible directed edges entering $x_q$. Their union is therefore
the set of vertices adjacent to $x_q$ in the graph obtained by disregarding
the orientation of $E^{\mathrm{vis}}_{K,g}$.
Equation~\eqref{eq:visibility_closure} repeatedly adds all such adjacent
vertices. Iteration to saturation consequently reaches exactly the connected
component containing the seed $x_i$.
\end{proof}

\begin{algorithm}[t]
\caption{Visibility-filtered round-trip closure for one point}
\label{alg:visibility_closure}
\begin{algorithmic}[1]
\Require directed KNN graph $\calG_K$, vote model $g$,
visibility thresholds $V_{K,g}^{(2)}$, point $x_i$
\State $C \gets \{x_i\}$
\Repeat
    \State $C_{\mathrm{old}}\gets C$
    \ForAll{$x_q\in C_{\mathrm{old}}$}
        \State $C\gets C\cup\calV_K^+(x_q)$
        \State $C\gets C\cup\calV_K^-(x_q)$
    \EndFor
\Until{$C=C_{\mathrm{old}}$}
\If{$|C|>1$}
    \State \Return $C$
\Else
    \State \Return rejected singleton closure
\EndIf
\end{algorithmic}
\end{algorithm}

\subsubsection{Optional structural pruning of saturated closures}
\label{sec:structural_pruning}

Before structural refinement, one visible path is sufficient to place its
endpoints in the same saturated closure. Structural pruning is an optional second stage,
enabled in the native executable by \texttt{--structural-pruning}, that tests
whether a large unpruned closure contains strongly supported internal
populations. It is disabled by default, so the reference inverse-square
regression and the four-kernel ablation remain unchanged. The native selector
\texttt{--vote-kernel} accepts \texttt{uniform}, \texttt{inv\_r},
\texttt{inv\_r2}, or \texttt{sigmoid}; \texttt{inv\_r2} remains the default,
and the sigmoid slope is controlled by \texttt{--sigmoid-alpha} (default $8$).
The selected kernel supplies both the popularity-modulated visible graph and
the popularity-ranked first mixture seed, while every subsequent pruning
criterion is held fixed.

Despite the term pruning, this stage does not rank or delete individual KNN
edges. It proposes a subdivision of an already formed connected component and
accepts that subdivision only when the statistical, size, and graph-coherence
criteria agree. The KNN adjacency lists themselves remain unchanged.

Let $C$ be an unpruned saturated closure. It is considered for pruning only
when
\begin{equation}
|C|\geq n_{\mathrm{parent}}
=
\max\{50,\lceil0.15N\rceil\}.
\label{eq:pruning_parent_size}
\end{equation}
Each coordinate is standardised within $C$, and diagonal-covariance Gaussian
mixtures with $m\in\{1,2,3,4\}$ components are fitted by
expectation--maximisation \citep{DempsterLairdRubin1977}. For dimension $d$,
the implementation compares
\begin{equation}
\operatorname{BIC}_m
=
-2\ell_m+p_m\log |C|,
\qquad
p_m=m(2d+1)-1,
\label{eq:pruning_bic}
\end{equation}
where $\ell_m$ is the highest fitted log likelihood among the deterministic
restarts and $p_m$ counts mixture
weights, means, and diagonal variances \citep{Schwarz1978}. Models with
$m>1$ use five deterministic restarts. In the first restart, the initial
centre is the point of greatest reverse-neighbour popularity and subsequent
centres are chosen by farthest-first traversal. The remaining restarts use
deterministic seeds, making repeated executions reproducible. Each restart
uses at most 200 EM iterations and stops when the change in total log
likelihood satisfies
$|\ell^{(t)}-\ell^{(t-1)}|\leq10^{-6}(1+|\ell^{(t)}|)$.
The variance updates add $10^{-6}$ in the standardised coordinates. These
settings define a numerical fit; a global likelihood maximum is not assumed.

Let $m^*$ denote the number of Gaussian-mixture components selected by BIC,
that is,
\begin{equation}
m^*=\argmin_m\operatorname{BIC}_m.
\end{equation}
A candidate is retained only if
\begin{equation}
m^*\in\{2,3\},
\qquad
\operatorname{BIC}_1-\operatorname{BIC}_{m^*}\geq10,
\label{eq:pruning_acceptance}
\end{equation}
and every proposed child contains at least
\begin{equation}
n_{\mathrm{child}}
=
\max\{20,\lceil0.05|C|\rceil\}
\label{eq:pruning_child_size}
\end{equation}
points. Thus, a four-component optimum is a conservative veto rather than an
instruction to subdivide a curved or structurally complex closure.

Mixture membership alone does not determine the output. Each proposed child
is decomposed into connected components in its induced symmetrised visible
subgraph. The largest component is retained as the principal component; hereafter,
we refer to any other connected component of that child as a \emph{fragment}.
Any fragment containing fewer than $n_{\mathrm{child}}$ points is reassigned
to whichever other proposed child has the greatest visible-arc support from
that fragment. Support counts incoming and outgoing visible directed arcs;
a reciprocal pair contributes twice. Ties prefer the more populous child,
then the smaller child index. The native implementation updates assignments
sequentially and recomputes the induced components over at most four passes.
Acceptance requires a pass with no reassignment, so every final child is
connected. The proposal is rejected if a non-principal component contains
at least $n_{\mathrm{child}}$ points, a child becomes empty, no unchanged pass
is reached within the limit, or a child violates the final size constraint.

For a candidate labelling $z_C$ that satisfies the mixture, size, and
graph-coherence criteria, define the crossing set
\begin{equation}
E_C^{\mathrm{cut}}
=
\left\{
\{x_i,x_j\}\in\overline E_{K,g}^{\mathrm{vis}}[C]:
z_C(x_i)\neq z_C(x_j)
\right\}.
\label{eq:pruning_cut}
\end{equation}
The subdivision is accepted only when
\begin{equation}
|E_C^{\mathrm{cut}}|>0.
\end{equation}
For a connected parent and a nontrivial subdivision, this positivity
condition follows automatically; it is an implementation consistency check,
not an additional measure of separation. Its regularised children then
replace $C$ in the final layer. The native
implementation does not modify the KNN adjacency lists; instead, it rebuilds
the cluster layer from the regularised child assignments. At the level of the
resulting partition, this is equivalent to treating the crossing connections
as cut. Reference labels play no role in proposal generation, model selection,
regularisation, or final reassignment.

\subsection{Cluster layers and multiscale evolution}
\label{sec:layers}

For a fixed $K$, the distinct accepted saturated closures define the
unpruned cluster layer
\begin{equation}
\calP_K=
\{C_1^{(K)},\ldots,C_{m_K}^{(K)}\},
\end{equation}
where $m_K$ is the number of distinct accepted closures at scale $K$.
Closures obtained from different seeds but corresponding to the same
symmetrised visible connected component are therefore represented only once.

When structural pruning is enabled, regularised accepted children replace
their parent closures in the final pruned layer. Unless stated otherwise,
the multiscale construction below refers to the unpruned layers. The
monotone-coarsening result does not extend automatically to independently
pruned layers, for which no nesting guarantee is claimed.

By varying $K$ over a prescribed set of scales
\begin{equation}
K\in
\{K_{\min},K_{\min}+1,\ldots,K_{\max}\},
\end{equation}
one obtains the unpruned multiscale sequence
\begin{equation}
\calP_{K_{\min}},
\calP_{K_{\min}+1},
\ldots,
\calP_{K_{\max}}.
\end{equation}

Each reported layer is computed independently for its own value of $K$, so
the KNN graph, popularity, and visibility are all defined at the same
neighbourhood scale. The implementation convention used to enforce this in the
experiments is described in Sect.~\ref{sec:experimental_protocol}.

\begin{proposition}[Monotone coarsening under reference inverse-square voting]
\label{prop:inverse_square_coarsening}
Assume $r_{ij}>0$ for $i\neq j$ and a fixed deterministic tie-breaking rule
such that
\begin{equation}
\calN_K^+(x_i)\subseteq\calN_{K+1}^+(x_i)
\label{eq:nested_knn}
\end{equation}
for every $x_i$. Under the unpruned reference inverse-square vote, write
\begin{equation}
E_K^{\mathrm{vis}}
=
E_{K,\mathrm{inv\_r2}}^{\mathrm{vis}},
\qquad
P_K
=
P_{K,\mathrm{inv\_r2}},
\qquad
V_K^{(2)}
=
V_{K,\mathrm{inv\_r2}}^{(2)}.
\label{eq:inverse_square_shorthand}
\end{equation}
Then
\begin{equation}
E_K^{\mathrm{vis}}
\subseteq
E_{K+1}^{\mathrm{vis}}.
\label{eq:nested_visible_edges}
\end{equation}
Consequently, since the unpruned saturated closures form proper partitions of \(X\),
\begin{equation}
\calP_K\preceq\calP_{K+1},
\label{eq:partition_refinement}
\end{equation}
where $\preceq$ denotes partition refinement in the sense that every cluster
of $\calP_K$ is contained in a unique cluster of $\calP_{K+1}$. Thus,
increasing $K$ can preserve or merge clusters but cannot split them.
\end{proposition}

\begin{proof}
For a fixed source $x_i$, order its squared neighbour distances as
\[
D_{i,(1)}
\leq
\cdots
\leq
D_{i,(K)}
\leq
D_{i,(K+1)}.
\]
It follows immediately from Eq.~\eqref{eq:squared_knn_radius} that
\[
R_{K+1}^{(2)}(x_i)
\geq
R_K^{(2)}(x_i).
\]
Moreover,
\[
\left[V_{K+1,0}^{(2)}(x_i)\right]^2
=
\frac{
K\left[V_{K,0}^{(2)}(x_i)\right]^2
+
D_{i,(K+1)}^2
}{K+1}.
\]
Since
\[
D_{i,(K+1)}^2
\geq
\frac{1}{K}
\sum_{j=1}^{K}D_{i,(j)}^2
=
\left[V_{K,0}^{(2)}(x_i)\right]^2,
\]
we obtain
\[
V_{K+1,0}^{(2)}(x_i)
\geq
V_{K,0}^{(2)}(x_i).
\]

Under reference inverse-square voting, a retained directed edge
$x_\ell\to x_i$ contributes
\[
\frac{1}{D_{\ell i}}
\]
to the popularity of $x_i$, independently of $K$. The nesting of the
forward KNN sets in Eq.~\eqref{eq:nested_knn} implies
\[
\calN_K^-(x_i)
\subseteq
\calN_{K+1}^-(x_i),
\]
and therefore
\[
P_{K+1}(x_i)\geq P_K(x_i).
\]
Both $P_K(x_i)$ and $R_K^{(2)}(x_i)$ are non-negative and non-decreasing
with $K$, so their product is also non-decreasing. Together with the
monotonicity of $V_{K,0}^{(2)}(x_i)$,
Eq.~\eqref{eq:popularity_modulated_visibility} gives
\[
V_{K+1}^{(2)}(x_i)
\geq
V_K^{(2)}(x_i).
\]

Now let
\[
(x_i,x_j)\in E_K^{\mathrm{vis}}.
\]
Then
\[
x_j\in\calN_K^+(x_i)
\subseteq
\calN_{K+1}^+(x_i)
\]
and
\[
D_{ij}
\leq
V_K^{(2)}(x_i)
\leq
V_{K+1}^{(2)}(x_i).
\]
Therefore
\[
(x_i,x_j)\in E_{K+1}^{\mathrm{vis}},
\]
which proves
\[
E_K^{\mathrm{vis}}
\subseteq
E_{K+1}^{\mathrm{vis}}.
\]

By Proposition~\ref{prop:closure_component}, the unpruned saturated closures
are the connected components of the symmetrised visible graph. Adding edges
can merge connected components but cannot split them, which proves the stated
refinement relation between successive proper partitions.
\end{proof}

\subsubsection{Inter-layer overlap and empirical nesting}

The overlap between clusters in consecutive layers is
\begin{equation}
W_{ab}^{(K,K')}
=
|C_a^{(K)}\cap C_b^{(K')}|.
\label{eq:overlap_count}
\end{equation}
Useful normalisations include
\begin{equation}
F_{ab}^{(K,K')}
=
\frac{
|C_a^{(K)}\cap C_b^{(K')}|
}{
|C_a^{(K)}|
},
\label{eq:frac_from}
\end{equation}
and the Jaccard index
\begin{equation}
J_{ab}^{(K,K')}
=
\frac{
|C_a^{(K)}\cap C_b^{(K')}|
}{
|C_a^{(K)}\cup C_b^{(K')}|
}.
\label{eq:jaccard}
\end{equation}

The overlap quantities above are general enough to detect preservation,
absorption, merging, or splitting between consecutive layers. In all
reported unpruned $K$-layer sequences, no partial split transition was
observed. Every cluster at scale $K$ was
either preserved or wholly contained in a unique cluster at scale $K+1$.
Thus, the observed sequences formed successive merge-only coarsenings.

For the reference inverse-square model, this nesting follows from
Proposition~\ref{prop:inverse_square_coarsening} whenever the reported layers
are proper partitions. For the other vote kernels, the same merge-only
behaviour is an empirical observation rather than a guaranteed property.

This inter-layer correspondence is distinct from intra-layer overlap, which
would indicate that a given layer is not a proper partition.

\subsubsection{Saturation scale}

The $K$-layer construction also provides a label-free description of
partition stability. We denote by $K_{\mathrm{sat}}$ the onset of a
persistent layer over a prescribed stability window. This quantity
identifies a stable connectivity regime; it does not by itself imply
agreement with an external reference partition.

Operationally, saturation can be detected by monitoring the overlap between
consecutive layers. If each cluster in $\calP_K$ maps entirely to a unique
cluster in $\calP_{K+1}$, and vice versa, the two layers are equivalent.
More robustly, one may require the same equivalence to persist over a short
stability window,
\begin{equation}
\calP_K
\equiv
\calP_{K+1}
\equiv
\cdots
\equiv
\calP_{K+h},
\label{eq:partition_stability_window}
\end{equation}
where $h\geq1$ is a prescribed integer persistence parameter and
$K+h\leq K_{\max}$. The saturation scale can
then be defined as
\begin{equation}
K_{\mathrm{sat}}
=
\min
\left\{
K:
\calP_K
\equiv
\calP_{K+1}
\equiv
\cdots
\equiv
\calP_{K+h}
\right\}.
\label{eq:ksat}
\end{equation}
The minimum is taken over tested scales satisfying
$K_{\min}\leq K\leq K_{\max}-h$; if none satisfies the condition,
$K_{\mathrm{sat}}$ is undefined for that window and tested range. A finite
plateau does not preclude a later merger. Here, $\equiv$ denotes equality
of cluster layers up to relabelling of cluster indices.

% \subsubsection{Inter-layer overlap and empirical nesting}

\subsection{Implementation agreement and partition diagnostic}\label{sec:partition_diagnostic}

An important consistency check is whether each accepted layer $\calP_K$ is a true partition of the dataset. For each $K$, this can be verified by comparing the sum of accepted cluster cardinalities with the cardinality of the union of all accepted cluster members:
\begin{equation}
\sum_{a=1}^{m_K}|C_a^{(K)}|=\left|\bigcup_{a=1}^{m_K}C_a^{(K)}\right|.
\label{eq:partition_check}
\end{equation}
When the equality in Eq.~\eqref{eq:partition_check} holds and the union contains all $N$ data points, the accepted layer is not merely a cover of the dataset, but a proper partition:
\begin{equation}
\bigcup_{a=1}^{m_K} C_a^{(K)}=X,
\qquad
C_a^{(K)}\cap C_b^{(K)}=\varnothing
\quad\text{for }a\neq b.
\end{equation}
This property is important because the multiscale object produced by the algorithm is then a sequence of valid partitions, rather than a sequence of overlapping covers. Inter-layer overlap is still expected and informative, but intra-layer overlap would indicate an invalid layer representation. If singleton closures are rejected, the diagnostic also verifies whether any points were left outside the accepted layer.

If the left-hand side of Eq.~\eqref{eq:partition_check} is larger than the right-hand side, at least one point belongs to more than one cluster in the same layer.

Before comparing vote kernels, we verified the Python implementation against
the unpruned inverse-square C implementation using identical feature columns,
eight datasets, and independent runs at $K=2,\ldots,16$ with $K_0=K$.
Across all 120 runs, the C and Python implementations assigned every point
to the same cluster, up to permutation of cluster identifiers, including the
unclustered masks. Six additional C--Python comparisons covered the three
weighted kernels on the variable-density blobs at $K=8$ and the bridge at
$K=6$; these unpruned partitions also assigned every point identically up to
cluster-label permutation. Together, these tests establish implementation
agreement under the experimental protocol used in this study.

Table~\ref{tab:partition_diagnostic} reports a representative subset for the variable-density blobs dataset. Since all $N=2000$ points are represented exactly once and the unclustered count is zero, these layers are proper partitions. ARI and NMI compare each partition with the known reference labels and are invariant under a permutation of cluster identifiers.

\begin{table}[t]
\caption{Unpruned inverse-square regression results for the variable-density blobs dataset. Each $K$ was run independently with $K_0=K$.}\label{tab:partition_diagnostic}
\centering
\begin{tabular}{rrrrr}
\toprule
$K$ & clusters & ARI & NMI & unclustered\\
\midrule
2 & 595 & 0.0116 & 0.3991 & 0\\
3 & 277 & 0.0403 & 0.4491 & 0\\
4 & 99  & 0.2809 & 0.5716 & 0\\
5 & 18  & 0.6661 & 0.8017 & 0\\
6 & 6   & 0.7774 & 0.8941 & 0\\
7 & 5   & 0.7805 & 0.8981 & 0\\
8 & 4   & 0.7817 & 0.9035 & 0\\
\bottomrule
\end{tabular}
\end{table}

This diagnostic is different from inter-layer overlap. Equation~
\eqref{eq:partition_check} tests whether a single reported layer is a proper
partition, whereas the inter-layer quantities in
Eqs.~\eqref{eq:overlap_count}--\eqref{eq:jaccard} describe how clusters
correspond across neighbourhood scales.

For the reference inverse-square model, the merge-only character of
successive proper unpruned layers follows from
Proposition~\ref{prop:inverse_square_coarsening}.
Figure~\ref{fig:multiscale_transition} illustrates this structure on the
variable-density blobs. In the displayed $K=5,6,7,8$ sequence, every cluster
at one scale is wholly contained in a unique cluster at the next scale, so
the transition forest contains preservation and merger events but no split.

The same pointwise audit shows that the unpruned partition at $K=8$ remains
unchanged through $K=16$. Nevertheless, internal stability is not equivalent
to external correctness. At $K=7$, the cluster count already equals the five
reference classes, but ARI is only $0.7805$; at the tested-range saturation
beginning at $K=8$, four components remain and ARI is $0.7817$. The structural
pruning experiment in Sect.~\ref{sec:structural_pruning_results} addresses
this specific stable merger without using the reference labels at execution
time.

\begin{figure}[t]
\centering
\includegraphics[width=4.8in,height=6.8in,keepaspectratio]{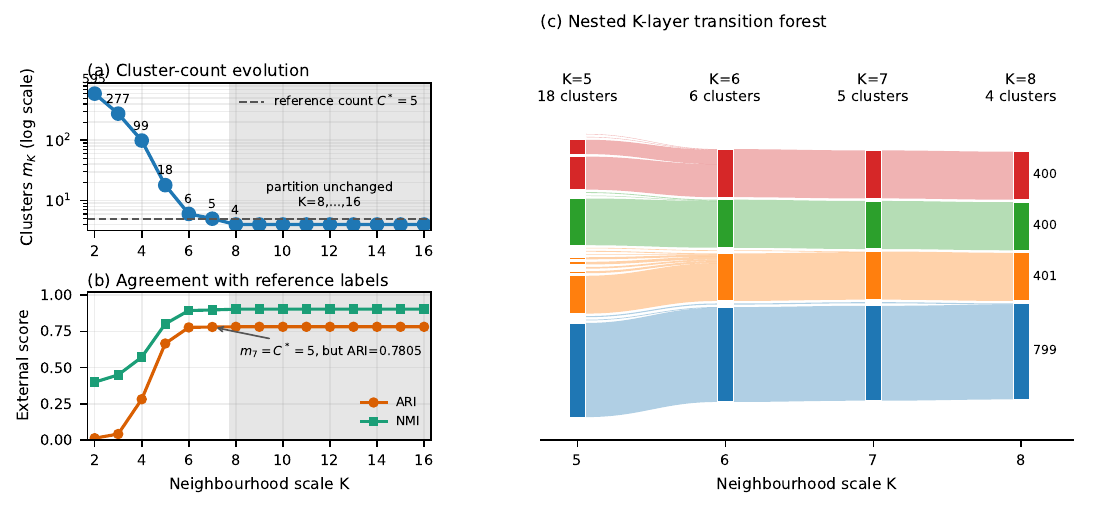}
\caption{Multiscale evolution of the unpruned reference inverse-square model
on the variable-density blobs.  Every layer is generated independently with
$K_0=K$.  (a) The cluster count decreases from 595 microclusters at $K=2$ to
four components at $K=8$ and is plotted on a logarithmic scale; the dashed
line is the reference count $C^*=5$, used only for evaluation.  (b) ARI and
NMI against the withheld reference labels.  The shaded interval marks the
exactly unchanged partition from $K=8$ through the largest tested value,
$K=16$.  (c) The merge-only transition forest over $K=5,6,7,8$.  Node and
ribbon thicknesses encode point population and overlap count
$W_{ab}^{(K,K')}$, respectively.  Colours identify the four lineages ending
at $K=8$ and do not encode reference labels.  The equality $m_7=C^*$ without
high ARI and the stable $m_K=4$ plateau demonstrate that cluster-count or
partition stability alone does not guarantee recovery of latent classes.}
\label{fig:multiscale_transition}
\end{figure}

\subsection{Scale, popularity, visibility, and saturation}\label{sec:scale}

The neighbourhood size $K$ controls the propagation scale of RTKNNC, while
the reverse-neighbour popularity $P_{K,g}(x_i)$ and the local visibility
threshold $V^{(2)}_{K,g}(x_i)$ regulate which KNN links are admissible.
Small values of $K$ favour local selectivity but may fragment the graph,
whereas larger values increase propagation. Popularity-modulated visibility
prevents this increase from making all KNN links equally admissible.

The base RMS threshold in Eq.~\eqref{eq:rms_visibility_base} provides a local
distance scale, while Eq.~\eqref{eq:popularity_modulated_visibility} enlarges
it according to incoming popularity. The vote kernel therefore changes how
incoming neighbour evidence affects visibility without changing the round-trip
closure operator itself. Under Eq.~\eqref{eq:nonempty_visible_neighbourhood},
singleton rejection is only an implementation safeguard and does not act as
an outlier filter. Across increasing $K$, the resulting layers may evolve from
local microfragmentation to stable larger-scale partitions; $K_{\mathrm{sat}}$
summarises the onset of such persistence.

\section{Experimental protocol}\label{sec:experimental_protocol}

The controlled benchmark reuses the eight synthetic datasets already distributed with the implementation:
\begin{itemize}
    \item Gaussian blobs with equal and variable densities;
    \item two moons and concentric circles;
    \item anisotropic Gaussian clusters;
    \item a uniform point cloud;
    \item two dense clusters connected by a sparse bridge;
    \item a Swiss-roll embedding.
\end{itemize}
The label column is used only for evaluation and is never included among the clustering features. The Swiss-roll labels are an artificial discretisation of a continuous manifold and its external scores are therefore interpreted descriptively. ARI and NMI are not reported for the uniform dataset because it has no nontrivial reference partition.

For every dataset, $K$ ranges from 2 through 16. Each $K$ is run independently with $K_0=K$. The KNN graph is constructed once for a dataset--$K$ pair and is reused unchanged by all four vote kernels. All models share the inverse-square source budget in Eq.~\eqref{eq:reference_source_budget} and therefore require strictly positive neighbour distances, so the input is checked for exact duplicate points before execution. Uniform, inverse-distance, inverse-square, and sigmoid refer to the four allocation shapes in Eq.~\eqref{eq:vote_kernel_family}; the sigmoid uses $\alpha=8$. RMS visibility and singleton rejection are enabled for every model. Structural pruning is disabled throughout this 840-run comparison, preserving a controlled ablation in which only the within-source allocation changes.

The independent-run convention is important for reproducibility. In the raw
interface, a run with $K_0<K_{\max}$ allocates the neighbour list at
$K_{\max}$ before the internal loop, so a construction-time radius may remain
tied to $K_{\max}$ at smaller intermediate layers. Running each reported scale
with $K_0=K$ ensures that the graph, radius, popularity, and visibility all
refer to the same stated neighbourhood scale.

The bridge benchmark varies the retained fraction of the existing labelled bridge points over $\{0,0.10,0.25,0.50,0.75,1\}$. Core points are never subsampled. The two dense cores are considered merged when their dominant predicted cluster identifier is the same; this criterion is independent of the numerical names assigned to clusters.

The output is one long-format CSV. For each dataset, $K$, and vote model, it records:
\begin{itemize}
    \item the dataset name, $K$, vote model, and complete model parameters;
    \item the number of clusters and unclustered points;
    \item ARI and NMI when a nontrivial reference partition is available \citep{HubertArabie1985,VinhEppsBailey2010};
    \item the mean number of visible neighbours and mean visibility in squared-distance units;
    \item end-to-end runtime, including KNN-graph construction;
    \item the bridge-merger indicator for the parameterised bridge rows.
\end{itemize}
The same measured graph-construction time is included in each of the four model runtimes for a given dataset--$K$ pair. This preserves end-to-end interpretation while ensuring that differences in graph construction cannot favour one vote kernel.

Structural pruning is evaluated separately in the native C executable under
the reference inverse-square vote.  Each case is run twice with identical
coordinates and isolated $K$: once with the default unpruned closure and once
with \texttt{--structural-pruning}.  The selected cases are variable-density
blobs at $K=8$, the complete sparse bridge at $K\in\{6,10\}$, and three
preservation controls---circles at $K=7$, two moons at $K=10$, and anisotropic
blobs at $K=8$.  The fixed implementation constants are one through four
diagonal mixture components, five restarts for $m>1$, a minimum BIC gain of
10, the parent and child sizes in Eqs.~\eqref{eq:pruning_parent_size} and
\eqref{eq:pruning_child_size}, and a four-component veto.  The executable is
passed \texttt{--dim 2}; the label column is read only afterwards to compute
ARI and NMI.  Complete native logs, layer files, input hashes, and a
machine-readable parameter manifest accompany the figures.

To test whether the two principal corrections depend on the reference allocation shape,
the variable-density blobs at $K=8$ and the complete bridge at $K=6$ are also
run natively with the three budget-matched shapes $1/u$, $1/u^2$, and
$S_\alpha(u)$, with $\alpha=8$.  Uniform allocation remains in the broader
unpruned ablation but is excluded from this focused weighted-kernel check.  For
each kernel, the unpruned C result must first match the corresponding Python
partition pointwise; the same native mixture, BIC, size, coherence, and pruning
parameters are then applied without retuning.

\subsection{External-method comparison protocol}\label{sec:external_comparison_protocol}

The final comparison uses the six datasets having nontrivial discrete
references and seven external methods: agglomerative clustering, $k$-means,
Gaussian mixtures, DBSCAN, OPTICS, KNN-affinity spectral clustering, and
HDBSCAN.  The uniform cloud has no reference partition, and the Swiss-roll
bands are an artificial discretisation of a continuous manifold; both are
therefore excluded from this table.  Every method receives the same raw feature columns, without
method-specific input preprocessing. RTKNNC standardises coordinates locally
inside each candidate parent only for its mixture-fitting stage, as described
in Sect.~\ref{sec:structural_pruning}; its KNN graph uses the raw features.
RTKNNC with structural refinement enabled uses the reference inverse-square vote and sweeps
$K=2,\ldots,16$ in the native executable.  Agglomerative clustering evaluates
Ward, complete, average, and single linkage, both without a connectivity
constraint and with symmetric KNN connectivity at 5, 10, and 15 neighbours.
$k$-means uses 20 initialisations; Gaussian mixtures use five initialisations
under full, tied, diagonal, and spherical covariance models.  These three
methods and spectral clustering receive the reference cluster count $C^*$.
Spectral clustering uses KNN affinity with $k_g=2,\ldots,16$, ARPACK
eigenvectors, and 20 label-assignment restarts at random seed zero.

DBSCAN sweeps $m\in\{3,5,10,20\}$ for \texttt{min\_samples}.  For each $m$,
let $d_i^{(m)}$ be the distance from $x_i$ to its $(m-1)$st proper neighbour,
because DBSCAN counts the query point itself.  Its $\epsilon$ candidates are
the 400 empirical quantiles of $\{d_i^{(m)}\}_{i=1}^N$ at
$q=0.0025,0.005,\ldots,1$.  OPTICS combines the same four values of $m$ with
$\xi\in\{0.01,0.025,0.05,0.1\}$ and minimum-cluster fractions
$\{0.01,0.025,0.05\}$.  HDBSCAN combines minimum cluster sizes
$\{5,10,20,50,100\}$, \texttt{min\_samples} in
$\{\texttt{None},5,10,20\}$, and excess-of-mass or leaf selection.  These
three density-based methods infer their reported cluster counts.

For each method--dataset pair, the table reports the candidate with maximum
ARI, then maximum NMI.  Remaining external-method ties prefer fewer noise
points and then the first candidate in the declared grid; RTKNNC ties prefer
smaller $K$. Thus, the result is a label-selected best-of-grid comparison
measuring what each preset parameter grid can represent. Reference labels
are used only to select the best candidate within each grid; ARI is not proposed
as an unsupervised deployment rule. A noise identifier $-1$ participates in ARI
and NMI as a predicted label, while the reported cluster count excludes it.
The comparison comprises 10,344 external fits and 90 RTKNNC fits.  Complete
external sweeps, selected rows, raw selected labels, software versions, input
hashes, and a provenance-checked table manifest accompany the article
artifacts.

\section{Results}\label{sec:results}

This section reports the empirical behaviour of RTKNNC under the experimental protocol described above. We first examine the controlled vote-kernel benchmark and then evaluate structural pruning, kernel sensitivity, and performance relative to the external clustering methods.

\subsection{Vote-model benchmark results}\label{sec:vote_results}

The complete sweep contains 840 benchmark runs: 480 runs from eight datasets, 15 values of $K$, and four kernels, plus 360 runs from six bridge fractions, 15 values of $K$, and four kernels. No benchmark run contained an unclustered point. Among the 210 matched dataset--$K$ comparisons across the two sweeps, the number of clusters differed across kernels in 60 combinations; ARI and NMI differed in 57 labelled combinations. The mean number of visible neighbours and the mean visibility both differed in all 210 combinations. The complete bridge occurs in both sweeps, so these totals include 15
repeated dataset--$K$ configurations (60 kernel runs). After counting each
configuration once, there are 195 configurations: cluster counts differ in
56, ARI and NMI in 53 labelled configurations, and both visibility summaries
in all 195. Thus, the allocation shape frequently changes the internal
visibility dynamics, although the resulting external optimum can remain stable.

Figure~\ref{fig:all_cluster_count_evolution} displays the complete cluster-count
trajectory for each primary dataset.  Every case begins with substantial
microfragmentation and then consolidates rapidly, but the macro-scale behaviour
is dataset dependent.  Some labelled datasets reach the reference count only
over a finite interval of $K$, whereas others stabilise at it or merge below it.
This complements the external scores: agreement in cluster count identifies a
candidate scale, but does not establish agreement with the reference partition.

\begin{figure}[p]
\centering
\includegraphics[width=4.8in,height=6.8in,keepaspectratio]{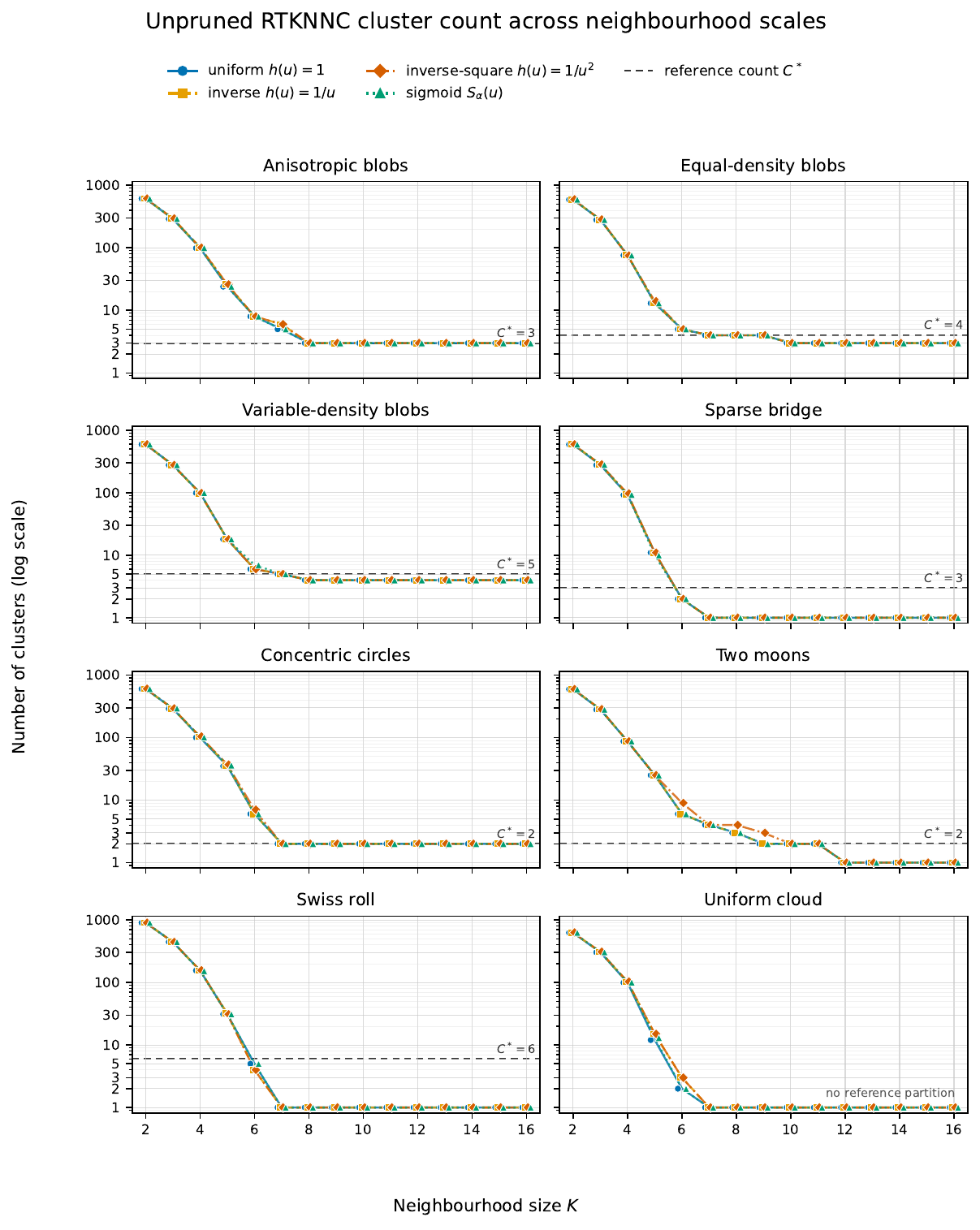}
\caption{Number of clusters as a function of $K$ for every primary dataset in
the controlled, unpruned four-kernel benchmark. Each point is obtained from an
independent execution with $K_0=K$, and the ordinate is logarithmic. Lines are
evaluated at the integer $K$ values; markers have small horizontal offsets only
to expose coincident counts from different kernels. Dashed lines show the
reference count $C^*$ for evaluation only. The Swiss-roll reference comprises
artificial bands, and the uniform cloud has no reference partition. Matching
$C^*$ is therefore neither an execution-time selection rule nor evidence of
pointwise partition correctness.}
\label{fig:all_cluster_count_evolution}
\end{figure}

For anisotropic clusters, equal-density blobs, circles, and moons, every kernel attained ARI=NMI=1 within the tested $K$ range. In this unpruned ablation, all kernels also attained the same best result for variable-density blobs (ARI $0.7817$, NMI $0.9035$, $K=8$) and for the labelled bridge dataset (ARI $0.8083$, NMI $0.8182$, $K=6$). These are unpruned optima, not the final results of the structural-pruning enhancement reported in Sect.~\ref{sec:structural_pruning_results}. For the discretised Swiss roll, the best scores were low, as expected for an artificial banding of a continuous manifold.  All optima occurred at $K=4$: uniform, inverse-distance, inverse-square, and sigmoid allocation attained ARI/NMI $0.1559/0.5160$, $0.1544/0.5141$, $0.1528/0.5131$, and $0.1558/0.5153$, respectively.

For the coloured comparison figures only, a common representative K is selected
for each labelled dataset by maximising the mean ARI across the four kernels,
with ties resolved in favour of the smaller K.  The same selected K is then
used in every predicted panel, so the graphical comparison does not grant a
different scale to each kernel.  This visualisation rule does not filter the
long-format results, which retain all $K=2,\ldots,16$ observations.

Figure~\ref{fig:moons_vote_models} provides a representative non-convex case in
which all four kernels recover the reference partition exactly.  The more
difficult variable-density case is reserved for the focused, pruned
three-kernel mosaic in Fig.~\ref{fig:blobs_pruning_vote_kernels}.

\begin{figure}[!tp]
\centering
\includegraphics[width=4.8in,height=6.8in,keepaspectratio]{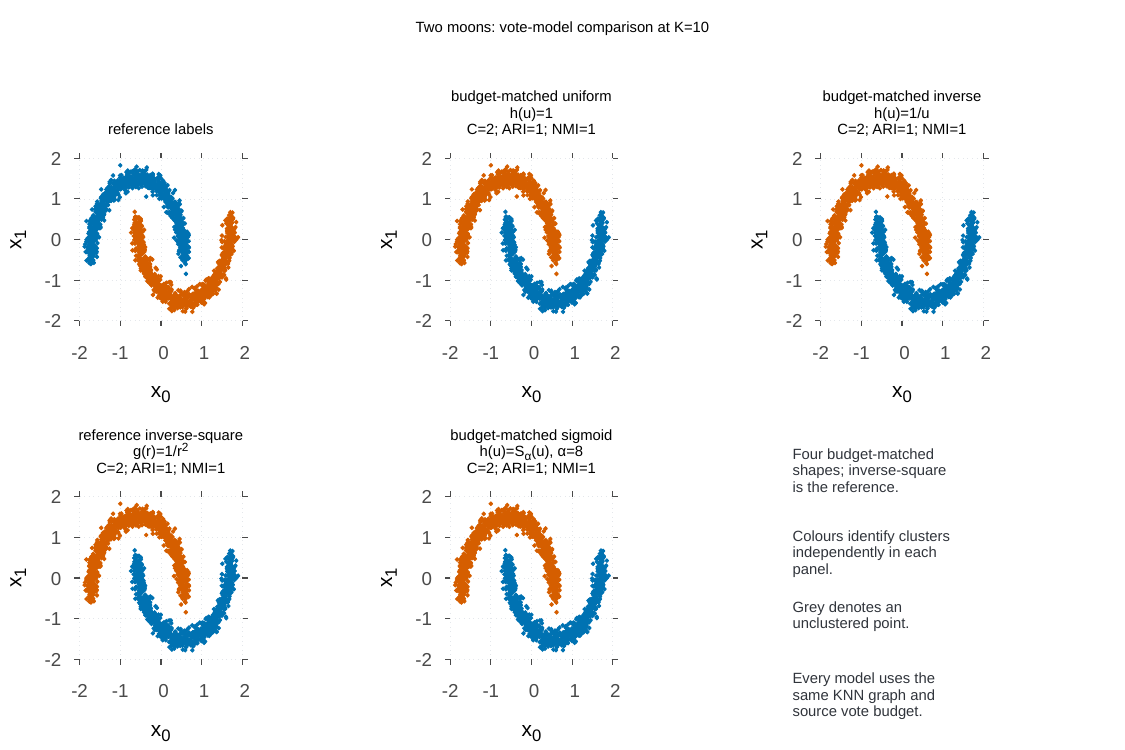}
\caption{Reference labels and the four predicted partitions for the two-moons
dataset at the common representative scale $K=10$: budget-matched uniform
$h(u)=1$, inverse-distance $h(u)=1/u$, reference inverse-square
$h(u)=1/u^2$, and local sigmoid $S_\alpha(u)$ with $\alpha=8$.  Every source
casts the common budget in Eq.~\eqref{eq:reference_source_budget}; the panels
differ only in its allocation among outgoing neighbours.  Every vote model recovers the two
non-convex components exactly (ARI${}={}$NMI${}={}$1). Cluster colours are
local to each panel and do not encode a cross-panel label correspondence.}
\label{fig:moons_vote_models}
\end{figure}

\subsubsection{Bridge-fraction transition}

Table~\ref{tab:bridge_transition} reports the first value of $K$ at which the two dense cores share a dominant predicted cluster. No kernel merged the disconnected cores when all bridge points were removed. At bridge fraction $0.10$, the reference inverse-square allocation merges one scale later than the other three; the boundaries coincide at the remaining tested fractions. This is direct evidence that the budget-matched shapes are not equivalent implementations.

\begin{table}[t]
\caption{First $K$ producing a merger of the two dense bridge cores for each of
the four vote kernels. A value $>16$ means that no merger occurred in the
tested range.}\label{tab:bridge_transition}
\centering
\begin{tabular}{rcccc}
\toprule
bridge fraction & $h(u)=1$ & $h(u)=1/u$ & $h(u)=1/u^2$ & $S_\alpha(u)$\\
\midrule
0.00 & $>16$ & $>16$ & $>16$ & $>16$\\
0.10 & 5 & 5 & 6 & 5\\
0.25 & 6 & 6 & 6 & 6\\
0.50 & 6 & 6 & 6 & 6\\
0.75 & 7 & 7 & 7 & 7\\
1.00 & 7 & 7 & 7 & 7\\
\bottomrule
\end{tabular}
\end{table}

For the complete bridge at $K=6$, the two reported components also do not
correspond to the three reference labels.  One contains all 900 points of the
left core, all 200 bridge points, and six peripheral points from the right
core; the other contains the remaining 894 right-core points.  The bridge has
therefore already been absorbed by the left component even though the two dense
cores have not yet fully merged.

\subsection{Native structural-pruning results}
\label{sec:structural_pruning_results}

Table~\ref{tab:native_structural_pruning} separates the native pruning
experiment from the four-kernel sweep.  The variable-density case changes from
four to five clusters at $K=8$, with ARI increasing from $0.7817$ to $0.9627$.
The accepted subdivision has 12 crossing visible connections. The five
inferred cluster sizes are $375$, $400$, $400$, $401$, and $424$; these sizes
show that matching the reference cluster count does not imply exact label
recovery.

For the complete sparse bridge, pruning recovers three components at both
tested scales.  At $K=6$, ARI increases from $0.8083$ to $0.9853$, and 10
visible connections cross the accepted subdivision.  At $K=10$, the unpruned
closure has already collapsed to one component, but an accepted three-child
subdivision has 32 crossing connections and reaches ARI $0.9865$.  Circles,
moons, and anisotropic blobs serve as preservation controls: no split is
accepted, and their ARI remains exactly $1$.

\begin{figure}[t]
\centering
\includegraphics[width=4.8in,height=6.8in,keepaspectratio]{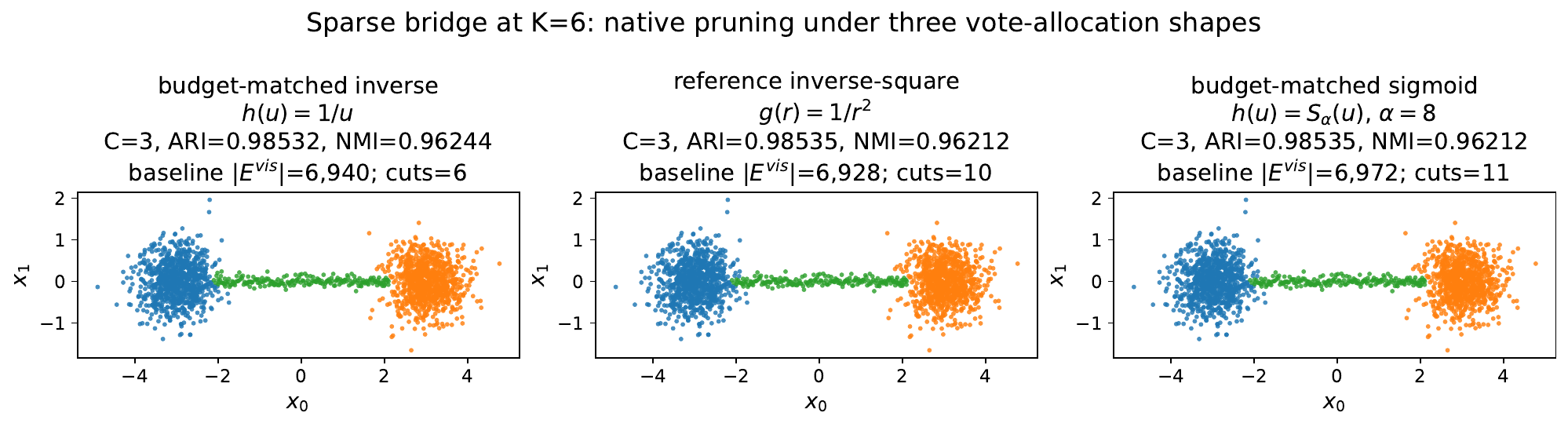}
\caption{Native structural pruning of the complete sparse bridge at $K=6$
under the three weighted allocation shapes.  The panels show budget-matched
$h(u)=1/u$, reference $h(u)=1/u^2$, and budget-matched local sigmoid
$S_\alpha(u)$ with $\alpha=8$, respectively.  Every model produces three
clusters.  Inverse-square and sigmoid allocation coincide pointwise (ARI
$0.98535$, NMI $0.96212$), while inverse-distance differs from them at four
points after label alignment (ARI $0.98532$, NMI $0.96244$).  The accepted
subdivisions have 6, 10, and 11 crossing visible connections, and their annotated native
unpruned visible-edge counts are 6940, 6928, and 6972, respectively.
The distinct unpruned visible graphs yield identical final partitions under
inverse-square and sigmoid allocation. Reference labels are used only for
external evaluation and colour alignment.}
\label{fig:bridge_pruning_vote_kernels}
\end{figure}

\begin{table}[t]
\caption{Native inverse-square results before and after structural pruning. $C^*$
is the number of reference classes, $C_0$ and $C_p$ are the unpruned and pruned
cluster counts, and $|E^{\mathrm{cut}}|$ is the number of crossing visible
connections in the accepted subdivision (zero when no split is accepted).
Labels are used only for the reported ARI and $C^*$.}
\label{tab:native_structural_pruning}
\centering
\small
\begin{tabular}{lrrrrrrr}
\toprule
dataset & $K$ & $C^*$ & $C_0$ & ARI$_0$ & $C_p$ & ARI$_p$ & $|E^{\mathrm{cut}}|$\\
\midrule
\input{native_structural_pruning_table_rows.tex}\\
\bottomrule
\end{tabular}
\end{table}

\subsubsection{Pruning sensitivity to the vote kernel}

Figures~\ref{fig:bridge_pruning_vote_kernels} and
\ref{fig:blobs_pruning_vote_kernels} address whether the successful split is
specific to the reference inverse-square vote.  For the variable-density
blobs, all three weighted allocation shapes start from four unpruned components and
produce the same five-cluster partition after pruning: ARI $0.9627$, NMI
$0.9614$.  Each accepted subdivision has 12 crossing visible connections.
The three partitions are pointwise
identical up to label permutation.  Their mean visible-neighbour counts are
not identical ($4.7010$, $4.6825$, and $4.7350$, respectively), confirming
that the kernels alter the visible graph.  In the companion Python edge audit,
whose partitions agree pointwise with the native implementation, the inverse-square directed
edge set differs from the inverse-distance and sigmoid sets by 89 and 269
visible edges, respectively.  These changes remain inside the same connected
components. The accepted pruning split also coincides empirically at this
scale; component agreement alone does not imply equality of mixture fits or
regularised subdivisions.

The bridge result is also stable in its scientific conclusion.  Every kernel
starts from two components and returns the intended three, with ARI between
$0.98532$ and $0.98535$.  Inverse-square and sigmoid allocation induce the same
partition, while inverse-distance differs by four peripheral assignments after
optimal label alignment.  Their accepted subdivisions have 10, 11, and 6 crossing connections,
respectively, when listed as inverse-square, sigmoid, and inverse-distance. In
the companion Python edge audit, the inverse-square unpruned visible edge set
differs by 56 and 190 directed edges from inverse-distance and sigmoid voting,
respectively.  The equality of cluster counts is
therefore supported by near-identical external partitions, not used as a
substitute for them.

\begin{figure}[t]
\centering
\includegraphics[width=4.8in,height=6.8in,keepaspectratio]{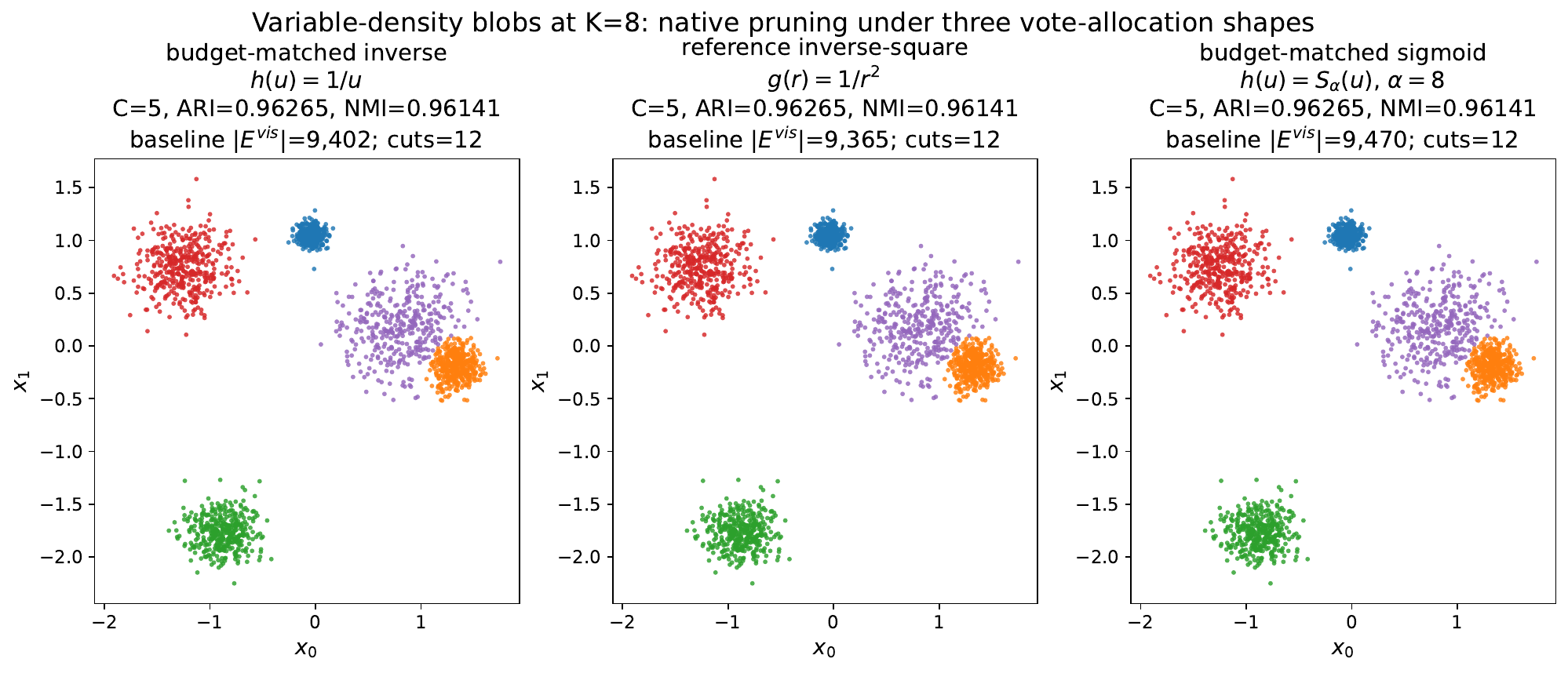}
\caption{Native structural pruning of the variable-density blobs at $K=8$
under the three weighted allocation shapes.  The three panels show, from left
to right, budget-matched $h(u)=1/u$, reference $h(u)=1/u^2$, and
budget-matched local sigmoid $S_\alpha(u)$ with $\alpha=8$.  All produce the
same pointwise five-cluster partition (ARI
$0.96265$, NMI $0.96141$) after structural refinement; each accepted
subdivision has 12 crossing visible connections.
The panel annotations expose the distinct native unpruned graphs:
$|E^{\mathrm{vis}}|=9402$, 9365, and 9470 directed edges, respectively.
Reference labels are omitted from the mosaic and are used only for colour
alignment and external scores; each unpruned C partition was first verified
pointwise against its Python kernel definition.}
\label{fig:blobs_pruning_vote_kernels}
\end{figure}

\subsection{Comparison with external clustering methods}
\label{sec:baseline_comparison}

Table~\ref{tab:all_methods_comparison} reports the label-selected best-of-grid
comparison defined in Sect.~\ref{sec:external_comparison_protocol}. RTKNNC attains
ARI${}={}$NMI${}={}$1 on anisotropic blobs, equal-density blobs, circles, and
two moons.  Several external methods tie these results, but the controls also
expose their inductive biases: $k$-means and Gaussian mixtures fail on circles
and two moons, whereas graph-, density-, and connectivity-constrained methods
recover the non-convex partitions. As expected from its nearest-centroid
geometry, $k$-means induces convex Voronoi cells, which cannot represent two
concentric rings adequately and only approximate the curved, interleaving
moons.  Consequently, it obtains ARI/NMI
$-0.0004/0.0001$ on circles and $0.4800/0.3815$ on two moons.  This limitation
persists even under the favourable protocol used here, in which $k$-means is
given the reference cluster count $C^*=2$.

The two connectivity-stress cases provide finer discrimination.  On the
variable-density blobs, RTKNNC returns the intended five clusters with ARI
$0.9627$ and NMI $0.9614$, narrowly below the full-covariance Gaussian-mixture
result $0.9639/0.9636$.  Spectral clustering reaches $0.9569/0.9581$, while
agglomerative clustering, $k$-means, OPTICS, and HDBSCAN obtain ARIs between
$0.9052$ and $0.9306$.  DBSCAN reaches ARI $0.7916$ and NMI $0.8044$; its
score-maximising candidate fragments the data into 15 clusters and marks 204
points as noise.

On the sparse bridge, RTKNNC attains the highest score,
$0.9865/0.9632$, followed closely by the full-covariance Gaussian mixture at
$0.9831/0.9552$.  Spectral clustering reaches $0.9487/0.8993$, and the other
external ARIs range from $0.6723$ to $0.9342$. The best-of-grid RTKNNC result
for the sparse bridge selects $K=11$ by a marginal ARI advantage over $K=10$;
the focused $K=6$ experiment remains useful because it isolates the earliest
corrected three-component partition under the three vote kernels.

\input{all_methods_table.tex}

The comparison therefore supports competitiveness rather than universal
superiority. RTKNNC matches the best ARI on four benchmark datasets. On the
variable-density dataset it obtains ARI $0.9627$, compared with $0.9639$ for
the full-covariance GMM, while on the sparse bridge it obtains the highest ARI,
$0.9865$, compared with $0.9831$ for the GMM. The fixed-count baselines have
the advantage of receiving $C^*$, whereas RTKNNC and the three density-based
methods do not. Future work will
extend the evaluation to repeated stochastic realisations, higher-dimensional
data, and label-free scale selection.

\section{Discussion}\label{sec:discussion}

RTKNNC combines established KNN-graph ingredients into a specific
visibility-filtered closure operator. Its defining step is recursive forward
and reverse-neighbour traversal under a popularity-modulated RMS visibility
filter, iterated to saturation. The optional,
popularity-seeded structural stage operates after closure and does not
replace the graph operator.

Before structural refinement, RTKNNC is defined by an algorithmic graph construction rather than by
minimising an explicit objective function. Changing only the vote allocation
altered the cluster count in 60 of 210 matched sweep comparisons and visibility in
all 210, although several datasets had the same best external partition under
all kernels. Vote allocation controls incoming evidence; the round-trip graph
and its closure control how visible evidence propagates. We report structural
pruning separately because its BIC criterion adds a local statistical model.

The inverse-square model defines the common source budget and is recovered
exactly by the budget-matched formulation.  A raw comparison against $1$,
$1/r$, and a bounded sigmoid would, however, mix kernel shape with incompatible
units and outgoing amplitudes; on these data, the raw sigmoid was consequently
almost indistinguishable from uniform voting.  The source-budget-matched family
removes that confound by conserving $B_K(x_i)$ for every source and changing
only how that mass is distributed among its $K$ outgoing edges.  It also makes
$P_{K,g}R_K^{(2)}$ invariant under global coordinate rescaling for every model,
not just inverse-square.  Feature-wise transformations can still change the
neighbour graph, so preprocessing remains part of the method specification.

On the two difficult datasets, the pruning result was not specific to
inverse-square voting. This does not imply kernel equivalence. For the
variable-density blobs, the three kernels
produce pointwise-identical final partitions even though their native mean
visible-neighbour counts differ and the companion audit finds hundreds of
kernel-dependent directed edges.  Those edge changes occur within the same
symmetrised connected components, and the mixture stage accepts the same
geometric subdivision.  For the bridge, the stability is only approximate:
inverse-distance voting changes four assignments relative to the pointwise-identical inverse-square and sigmoid results, a distinction hidden by
three-decimal ARI and NMI. For the variable-density blobs, these results identify a range in which
different visible graphs still lead to the same final partition, while the
bridge shows that small kernel effects remain measurable. These experiments establish robustness at
the reported scales, not general kernel invariance.

The bridge-fraction experiment provides a local sensitivity check. At retained bridge
fraction $0.10$, uniform, inverse-distance, and sigmoid allocation first merge
the dense cores at $K=5$, whereas reference inverse-square first merges them
at $K=6$.  This observed boundary difference confirms that equal results at
the complete-bridge reference scale do not imply kernel equivalence.  A denser
sweep over bridge geometry, noise, separation, and sigmoid slope may expose
further transition differences.

\subsection{Connectivity failure and structural correction}

The observed failures follow directly from the topology of the closure
operator. Recall the visible directed edge set
$E^{\mathrm{vis}}_{K,g}$ from Eq.~\eqref{eq:visible_edge_set} and its
symmetrised counterpart
$\overline E^{\mathrm{vis}}_{K,g}$ from
Eq.~\eqref{eq:symmetrised_visible_edge_set}.
By Proposition~\ref{prop:closure_component}, and since
Eq.~\eqref{eq:nonempty_visible_neighbourhood} ensures that no saturated
closure is a singleton under the stated assumptions, the unpruned saturated
closures are exactly the connected components of this symmetrised visible graph.
Before structural refinement, RTKNNC
can therefore recover a reference partition exactly only at a scale for
which its classes occupy separate visible connected components.

The enhanced method relaxes this connectivity restriction in a controlled
way. Eqs.~\eqref{eq:pruning_acceptance}--\eqref{eq:pruning_cut} permit a
large connected component to be subdivided only when the mixture, size, and
visible-graph coherence criteria agree. The structural stage therefore does
not alter the KNN construction or redefine visibility; it acts only after
the unpruned closure has been formed.

This condition explains why nonconvexity alone does not cause failure.
Although moons and circles are nonconvex, at a suitable value of $K$ the
low-density regions between their sampled components leave no visible path connecting them in the symmetrised visible KNN graph. The variable-density example
is different: the compact and diffuse right-hand Gaussians overlap in the
observed space. Points in the overlap create cross-class nearest-neighbour
edges, and an admissible chain places the two latent classes inside the same
unpruned saturated closure.

For the reference inverse-square model,
Proposition~\ref{prop:inverse_square_coarsening} further shows that increasing
$K$ cannot remove an already visible edge. Thus, once those two populations
belong to the same connected component at a given scale, increasing $K$ cannot
separate them again in the unpruned layer. Reducing $K$ may break the
connection, but in the variable-density experiment it produces
microfragmentation before recovering the latent separation. At $K=8$, the
structural stage instead detects two supported populations inside the merged
closure: BIC favours a two-component subdivision, and both proposed children
satisfy the visible-graph coherence requirements. This restores five clusters
with ARI $0.9627$. The remaining disagreement is consistent with the ambiguity
created by overlapping coordinates, which do not define a uniquely recoverable
boundary.

The sparse bridge creates the analogous connectivity failure through a
filament rather than direct overlap. At $K=6$, the bridge has already been
absorbed into the left-core component, while most of the right core remains a
second component. Because unpruned round-trip clustering operates through
connected closure, one admissible cross-boundary path is sufficient to join
all vertices reachable through that path. The structural stage splits the left-core--bridge component into two
children while retaining the other component, yielding three final clusters
at this scale. The accepted subdivision has 10 crossing visible connections
and preserves the bridge as a third structure,
raising ARI to $0.9853$. At $K=10$, even the unpruned partition has collapsed
to one component, but the accepted three-child subdivision restores a
three-cluster solution with ARI $0.9865$.

The Swiss roll presents a complementary case. Its reference bands discretise
one continuous manifold, so selective visibility yields many local connected
components rather than the prescribed bands, whereas less selective scales
eventually connect across them. The four-component veto and graph-coherence
test are intended to protect such curved or extensively fragmented structures,
but a systematic manifold study remains necessary.

The unpruned accepted closures are therefore induced by visible KNN
connectivity and form a partition under the stated assumptions, as
confirmed by the implementation checks in the reported experiments. The enhanced output is obtained after a
conservative local model-based test of the large unpruned components.

The present experiments establish the behaviour of RTKNNC under controlled
synthetic conditions and isolate the roles of vote allocation, visible-graph
connectivity, and structural pruning. Further work may extend this analysis to
repeated stochastic realisations, higher-dimensional datasets, additional
graph-based baselines, alternative covariance models, and label-free
strategies for selecting the neighbourhood scale. A broader parameter study
of bridge geometry and sigmoid slope may also help characterise the transition
boundaries of the different vote kernels.

\section{Conclusion}\label{sec:conclusion}

We presented Round-Trip KNN Clustering, a method that combines directed KNN
neighbourhoods, reverse-neighbour voting, and local visibility to identify
cluster structure across neighbourhood scales. We also introduced an optional
structural refinement that can separate supported populations when sparse
bridges or local overlap connect them into a single component. Across eight datasets and all
15 tested values of $K$ from 2 to 16, the reference C implementation and its
Python counterpart produced identical cluster partitions. The 840-run kernel
sweep then showed that vote allocation often changes visibility and sometimes
the induced partition.  The native enhancement then corrected the
central connectivity failures: variable-density blobs changed from four to
five clusters with ARI $0.7817\to0.9627$, and the bridge reached three clusters
with ARI $0.9853$ at $K=6$ and $0.9865$ at $K=10$.  It accepted no split on
three already correct controls, whose ARI remained $1$.  At the two difficult
reference scales, inverse-distance, inverse-square, and sigmoid voting all
preserved the corrected cluster counts with essentially identical scores, and
their six unpruned native partitions passed pointwise C--Python agreement.
A label-selected best-of-grid comparison comprising 10,434 candidate fits
across RTKNNC and seven external methods showed that RTKNNC with structural
refinement enabled matches ARI $1$ on four benchmark datasets, obtains ARI
$0.9627$ on the variable-density dataset versus $0.9639$ for a full-covariance
GMM, and attains the highest sparse-bridge ARI, $0.9865$ versus $0.9831$ for
the GMM.  These results separate the roles of vote weighting, visible connectivity, and
structural refinement in the behaviour of RTKNNC. Future work will examine
repeated stochastic realisations and higher-dimensional data, as well as
alternative covariance models, label-free scale selection, and additional
graph-based baselines.

\bmhead{Statements and Declarations}

\noindent\textbf{Competing interests.}
The authors declare no competing interests.

\medskip
\noindent\textbf{Ethics approval and consent to participate.}
Not applicable.

\medskip
\noindent\textbf{Consent for publication.}
Not applicable.

\medskip
\bmhead{Data and code availability}
The synthetic datasets used in this study, including their reference labels, and the custom C and Python implementations are publicly available in the project repository at \url{https://github.com/epmarinho/rtknnc}.

\bibliography{references}
\end{document}

%% file: native_structural_pruning_table_rows.tex
Variable-density blobs & 8 & 5 & 4 & 0.7817 & 5 & 0.9627 & 12 \\
Sparse bridge & 6 & 3 & 2 & 0.8083 & 3 & 0.9853 & 10 \\
Sparse bridge & 10 & 3 & 1 & 0.0000 & 3 & 0.9865 & 32 \\
Circles & 7 & 2 & 2 & 1.0000 & 2 & 1.0000 & 0 \\
Two moons & 10 & 2 & 2 & 1.0000 & 2 & 1.0000 & 0 \\
Anisotropic blobs & 8 & 3 & 3 & 1.0000 & 3 & 1.0000 & 0

%% file: all_methods_table.tex
\begin{table*}[t]
\caption{Best label-selected grid result for RTKNNC with structural refinement enabled and
external clustering methods. Each cell reports ARI/NMI; bold denotes the
highest ARI in the column, including ties. Methods marked $\dagger$ receive the
reference cluster count $C^*$; the others do not. Reference labels are used
only for scoring and best-of-grid selection, never as clustering features.}
\label{tab:all_methods_comparison}
\centering
\resizebox{\textwidth}{!}{%
\begin{tabular}{l*{6}{c}}
\toprule
Method & Anisotropic & Equal-density blobs & Variable-density blobs &
Sparse bridge & Circles & Two moons \\
\midrule
$C^*$ & 3 & 4 & 5 & 3 & 2 & 2 \\
\midrule
RTKNNC
& \textbf{1.0000/1.0000}
& \textbf{1.0000/1.0000}
& 0.9627/0.9614
& \textbf{0.9865/0.9632}
& \textbf{1.0000/1.0000}
& \textbf{1.0000/1.0000} \\

Agglom.$^\dagger$
& \textbf{1.0000/1.0000}
& \textbf{1.0000/1.0000}
& 0.9306/0.9444
& 0.9153/0.8580
& \textbf{1.0000/1.0000}
& \textbf{1.0000/1.0000} \\

K-means$^\dagger$
& \textbf{1.0000/1.0000}
& \textbf{1.0000/1.0000}
& 0.9189/0.9338
& 0.9342/0.8780
& -0.0004/0.0001
& 0.4800/0.3815 \\

GMM$^\dagger$
& \textbf{1.0000/1.0000}
& \textbf{1.0000/1.0000}
& \textbf{0.9639/0.9636}
& 0.9831/0.9552
& -0.0004/0.0001
& 0.5096/0.4080 \\

DBSCAN
& \textbf{1.0000/1.0000}
& \textbf{1.0000/1.0000}
& 0.7916/0.8044
& 0.8140/0.7422
& \textbf{1.0000/1.0000}
& \textbf{1.0000/1.0000} \\

OPTICS
& \textbf{1.0000/1.0000}
& 0.9987/0.9977
& 0.9266/0.9262
& 0.6723/0.6491
& \textbf{1.0000/1.0000}
& 0.9990/0.9972 \\

Spectral$^\dagger$
& \textbf{1.0000/1.0000}
& \textbf{1.0000/1.0000}
& 0.9569/0.9581
& 0.9487/0.8993
& \textbf{1.0000/1.0000}
& \textbf{1.0000/1.0000} \\

HDBSCAN
& \textbf{1.0000/1.0000}
& 0.9940/0.9902
& 0.9052/0.9075
& 0.8641/0.7828
& \textbf{1.0000/1.0000}
& \textbf{1.0000/1.0000} \\
\bottomrule
\end{tabular}%
}
\end{table*}